\documentclass{article}
\usepackage{ijcai26}

\usepackage{times}
\usepackage{soul}
\usepackage{url}
\usepackage[hidelinks]{hyperref}
\usepackage[utf8]{inputenc}
\usepackage[small]{caption}
\usepackage{graphicx}
\usepackage{amsmath}
\usepackage{amsthm}
\usepackage{booktabs}
\usepackage{algorithm}
\usepackage[switch]{lineno}

\usepackage{todonotes}
\usepackage{amsmath, amssymb, amsthm}
\usepackage{paralist}
\usepackage{xspace}
\usepackage{tikz}
\usetikzlibrary{calc,shapes,arrows}
\usepackage{thmtools, thm-restate}
\usepackage[noend]{algpseudocode}

\usepackage{caption}

\title{Monitoring Data-aware Temporal Properties (Extended Version)}
\author{
Alessandro Gianola$^1$
\and
Marco Montali$^2$\and
Sarah Winkler$^2$\\
\affiliations
$^1$INESC-ID/Instituto Superior Técnico, Universidade de Lisboa, Portugal\\
$^2$Free University of Bozen-Bolzano, Italy\\
\emails
alessandro.gianola@tecnico.ulisboa.pt,
\{montali, winkler\}@inf.unibz.it
}

\newcommand{\longversion}[2]{#1} 

\newcommand{\TT}{\mc T} 
\newcommand{\LL}{\mc L} 
\renewcommand{\SS}{\mc S} 
\newcommand{\PP}{\mc P} 
\newcommand{\FF}{\mc F} 
\newcommand{\eval}[2][i]{[#2]^{#1}_{\tau,\xi}} 
\newcommand{\LTLf}{LTL$_f$\xspace}
\newcommand{\LTLfMT}
{\smash{LTL$_f^{MT}$}}
\newcommand{\dataLTLf}{DB-LTL$_f$\xspace}

\newcommand{\PS}{\text{\textsc{ps}}}
\newcommand{\PV}{\textsc{pv}}
\newcommand{\CS}{\textsc{cs}}
\newcommand{\CV}{\textsc{cv}}

\newcommand{\mc}[1]{\mathcal{#1}}
\renewcommand{\vec}[1]{\overline{#1}}

\newcommand{\combine}[2]{[{#1}\,{\circledast}\,{#2}]} 
\newcommand{\U}{\mathrel{\mathsf{U}}} 
\newcommand{\G}{\mathsf{G}\xspace}
\newcommand{\F}{\mathsf{F}\xspace}
\newcommand{\R}{\mathrel{\mathsf{R}}}
\newcommand{\X}{\mathsf{X}\xspace}
\newcommand{\wX}{\mathsf{X}_w\xspace}
\newcommand{\sX}{\mathsf{X}_s\xspace}
\newcommand{\NFA}[1][\psi]{\mathcal N_{#1}}

\newcommand{\CG}{\textup{CG}} 
\newcommand{\regress}{\mathit{regress}} 
\newcommand{\precond}{\textsc{rc}} 
 
\newcommand{\goto}[1]{\mathrel{\raisebox{-2pt}{$\xrightarrow{#1}$}}}
\newcommand{\gotos}[1]{\to_{#1}^*}
\newcommand{\inn}{\,{\in}\,} 

\newcommand{\inquotes}[1]{#1} 
\newcommand{\prev}[1]{\lhd{#1}}
\newcommand{\Vpre}{V^\lhd}
\newcommand{\nolambda}[1]{[{#1}]}

\newcommand{\m}[1]{\mathsf{#1}}
\newcommand{\sat}{\textsc{ExtSat}}
\newcommand{\unsat}{\textsc{ExtViol}}
\newcommand{\llbracket}{[\![}
\newcommand{\rrbracket}{]\!]}

\makeatletter
\newcommand*{\owedge}{%
  \sbox0{$\oplus\m@th$}%
  \dimen2=.5\dimexpr\wd0-\ht0-\dp0\relax 
  \dimen@=\dimexpr\ht0+\dp0\relax
  \def\lw{.04}
  \def\radius{.5-\lw/2}%
  \kern\dimen2 
  \tikz[
    line width=\lw\dimen@,
    line join=round,
    x=\dimen@,
    y=\dimen@,
    baseline=\dimexpr-.5\dimen@+\dp0\relax,
  ]
  \draw
    (0,0) circle[radius=\radius]
    (225:\radius) -- (0,.5-\lw) -- (-45:\radius)
  ;%
  \kern\dimen2 
}
\newcommand*{\ovee}{%
  \sbox0{$\oplus\m@th$}%
  \dimen2=.5\dimexpr\wd0-\ht0-\dp0\relax 
  \dimen@=\dimexpr\ht0+\dp0\relax
  \def\lw{.04}
  \def\radius{.5-\lw/2}%
  \kern\dimen2 
  \tikz[
    line width=\lw\dimen@,
    line join=round,
    x=\dimen@,
    y=\dimen@,
    baseline=\dimexpr-.5\dimen@+\dp0\relax,
  ]
  \draw
    (0,0) circle[radius=\radius]
    (-225:\radius) -- (0,-.5+\lw) -- (45:\radius)
  ;%
  \kern\dimen2 
}
\makeatother

\newcommand{\domino}[2]{\text{\tikz[baseline=-0.5ex]{\node[scale=.7, inner sep=0pt]{$%
\left\{\begin{array}{@{\,}l@{=}l@{\,}}t&{#1}\\b&{#2}\\\end{array}\right\}$}}}}

\renewcommand{\sim}{\equiv_{\TT^*}}

\tikzstyle{state}=[draw, circle, inner sep=1.5pt, line width=.7pt, scale=.6]
\tikzstyle{edge}=[draw, ->, line width=.5pt]
\tikzstyle{action}=[scale=.55]
\tikzstyle{caption}=[scale=.9]
\tikzstyle{node} = [draw,rectangle split, rectangle split parts=2,rectangle split horizontal, rectangle split draw splits=true, inner sep=3pt, scale=.65, rounded corners=2pt]
\tikzstyle{goto} = [->]
\tikzstyle{action}=[scale=.6, black]
\tikzstyle{final}=[double]
\tikzstyle{pscolor}=[fill=green!30]
\tikzstyle{cscolor}=[fill=cyan!30]
\tikzstyle{pvcolor}=[fill=red!30]
\tikzstyle{cvcolor}=[fill=orange!30]

\newtheorem{theorem}{Theorem}

\newtheorem{example}[theorem]{Example}
\theoremstyle{definition}
\newtheorem{definition}[theorem]{Definition}
\newtheorem{lemma}[theorem]{Lemma}
\newtheorem{corollary}[theorem]{Corollary}

\newtheorem{assumptions}[theorem]{Assumptions}

\newcommand{\ltlf}{\textsc{LTL}$_f$\xspace}
\newcommand{\tool}{\textsc{MonThe}\xspace}

\newcommand{\trace}[1][n-1]{(M, \langle\alpha_0,\dots,\alpha_{#1}\rangle)}
\newcommand{\Mcarrier}{|M|} 

\DeclareMathOperator\wnextvar{\bigcirc\kern-1.17em\sim\kern0.2em}

\renewcommand{\phi}{\varphi}

\newcommand{\qfLTLfMT}{data\LTLf}

\begin{document}
\maketitle

\begin{abstract}
Dynamic systems in AI are often complex and heterogeneous, so that an internal specification is not accessible and verification techniques such as model checking are not applicable.
Monitoring is in such cases an attractive alternative, as it evaluates desirable properties along traces generated by an unknown dynamic system. In this work, we consider anticipatory monitoring of linear-time properties enriched with an arbitrary SMT theory over finite traces (\LTLfMT).
Anticipatory monitoring in this setting is highly challenging,
as the monitoring state depends on both the trace prefix seen so far and all its possible finite continuations.
Under reasonable assumptions on the background theory, we present and formally prove the correctness of a novel foundational framework for monitoring properties in an expressive fragment of \LTLfMT. The framework combines automata-theoretic methods to handle the temporal aspects of the logic, with automated reasoning techniques to address the first-order dimension.
Moreover, we identify for the first time decidable fragments of this monitoring problem that are practically relevant as they combine linear arithmetic with uninterpreted functions, which covers e.g. data-aware business processes and dynamic systems operating over a read-only database.
Feasibility is witnessed by a prototype implementation and preliminary evaluation. 
\end{abstract}

\section{Introduction}

Complex, dynamic AI systems increasingly comprise autonomous agents and heterogeneous components with unknown, inaccessible, or opaque internal specifications. This hampers the possibility of ascertaining their safety and trustworthiness through traditional design-time verification methods such as model checking and testing.
A remedy in such situations is to keep track of desired
properties at runtime, by \emph{monitoring} the executions of the black-box system under scrutiny.
A widely known, solid approach to monitoring is that of \emph{runtime verification} (RV), where given a logical property, a corresponding monitor can be automatically constructed, with the guarantee that it will
emit a provably correct verdict~\cite{LeuS09}.
These advantages have fostered the application of RV in various AI contexts, including multiagent systems~\cite{DaTY18,AlechinaHKL18}, robotics~\cite{BegemannKLS23}, cyber-physical systems~\cite{TracyTVS20}, and business processes~\cite{LMMR15,CDMP22}. The vast majority of such approaches specify properties using variants of linear-time temporal logic (LTL). 

\citeauthor{LeuS09} (\citeyear{LeuS09}) point out two key insights in monitoring:
First, since the trace $\tau$ seen so far is necessarily of finite length, this calls for a finite-trace semantics for LTL.
However, as the system evolves, $\tau$ is only the prefix of a (yet unknown) continuation. Different variants of LTL then have to be studied depending on whether such a continuation eventually terminates, or continues forever. 
In many AI-relevant settings (e.g., planning and business processes \cite{BacK00,BaBM09,DeGV13,GMM14}) the monitored system operates for an unbounded, yet finite, number of steps. This means that also each one of the (infinitely many) possible suffixes of $\tau$ has a finite length, making \ltlf (LTL over finite traces \cite{DeGV13}) a natural candidate logic for monitoring \cite{DFIP20,DDMM22,CDMP22}.

A key limitation of 
many of these
approaches 
is that \ltlf is propositional and cannot express properties about structured data. This is ubiquitous in AI systems, as witnessed by the new agentic wave, where agents are often delegated to execute business tasks and processes, and have consequently to deal with different types of structured data, such as numbers with arithmetic, sets, lists, and uninterpreted functions. Recently, ~\citeauthor{GGG22} (\citeyear{GGG22}) brought forward \LTLf modulo theories (\LTLfMT) to enrich \LTLf with the full spectrum of theories supported by state-of-the-art SMT \cite{BarrettSST21} solvers, which include all the datatypes mentioned before. The following example illustrates the need for such rich properties:

\newcommand{\myc}{\mathtt{myc}}
\newcommand{\tvar}{t}
\newcommand{\cvar}{c}
\newcommand{\bvar}{b}
\newcommand{\ttype}{\mathit{Ticket}}
\newcommand{\ctype}{\mathit{Concert}}
\newcommand{\cfun}{\mathop{con}}
\newcommand{\pfun}{\mathop{price}}
\renewcommand{\undef}{\mathtt{undef}}
\begin{example}
\label{ex:ticket}
\label{ex:concert}
Consider an online marketplace for concert tickets. A buyer agent dynamically receives ticket offers, and is tasked by its human principal with the goal of selecting the best offer (that is, the ticket with the lowest price) for a given concert (indicated by constant $\myc$ below). Every ticket has two properties: its price, and the concert it refers to. 

Formally, this setting is captured by three datatypes $\ttype$, $\ctype$ (with $\myc \inn \ctype$), and the real numbers $\mathbb{R}$, a dedicated constant $\undef$ of type $\ttype$ and  functions
$\pfun\colon\ttype\to\mathbb{R}$ 
and $\cfun\colon\ttype\to \ctype$
that resp.~assign to each ticket its price and the concert it is for. Such two functions conceptually model a database table with three columns, where each record stores a ticket id (primary key) with its price and concert.
The agent should select the cheapest ticket. An execution consists of a trace where at each time step the values of two variables of type $\ttype$ are specified:
$\tvar$, the ticket currently offered by the marketplace; and
$\bvar$, the ticket currently bookmarked by the agent. 

Since the agent specification is not accessible, its human principal wants to formalise the expected behavior and check whether the agent indeed adheres to it. To this end, the principal constructs a suitable \LTLfMT formula composed of different parts. The first part $\psi_{change}$ captures the situation where the bookmarked ticket must be changed to the currently received one, because it has a lower price.
The notation $\prev \bvar$ is used here to refer to the previous value of variable $\bvar$, i.e., the value $\bvar$ held in the instant before the current one. With this notation at hand, we set $\psi_{change}$ as

\noindent
\begin{footnotesize}
$\cfun(\tvar)\,{=}\,\myc \land (\prev{\bvar}\,{=}\,\undef \lor \pfun(\tvar) < \pfun(\prev{\bvar})) \to \bvar\,{=}\,\tvar$
\end{footnotesize}

\noindent
Next, the formula $\psi_{keep}$ deals with the case where the bookmarked ticket is kept, since the currently received ticket either refers to a concert of no interest, or has a higher price than the bookmarked one. Formally, $\psi_{keep}$ is given by

\noindent
\begin{footnotesize}
$(\cfun(\tvar)\,{\neq}\,\myc \lor (\prev\bvar\,{\neq}\,\undef \land \pfun(\tvar){>}\pfun(\prev{\bvar}))) \to \bvar{=}\prev{\bvar} 
$
\end{footnotesize}

To bundle everything together, we assume for simplicity that the ticket stream starts at the second instant. We can hence require that $\bvar$ is initially $\undef$, while the formulae $\psi_{keep}$ and $\psi_{change}$ must hold in every instant of the trace, starting from the next one.\footnote{We assume for simplicity that if a ticket with the same price appears, the agent keeps the one already bookmarked.
}
The latter condition is captured using the \ltlf operators $\X$ and $\G$, yielding
\[\psi := (\bvar\,{=}\,\undef) \wedge \X \G (\psi_{change} \wedge \psi_{keep})\]
\end{example}

Differently from the simpler settings of propositional \ltlf \cite{DDMM22}, or of \ltlf extended with numerical variables \cite{FMPW23}, where a trace is simply a sequence of assignments, dealing with \LTLfMT also requires to keep track of a first-order model $M$. In Ex.~\ref{ex:concert}, this corresponds to the set of tickets seen in the trace, as well as their prices and concerts. 
Therefore, different interpretations of the monitoring problem are conceivable: the model underlying a trace can be considered as fixed, or it may be assumed to change in every step.
In this work, we adopt the following:  
given a trace $\tau$ with model $M$, a continuation of $\tau$ provides an \emph{extension of $M$} -- a notion that we later anchor precisely to the  model-theoretic concept of \emph{embeddings}. This interpretation has the following advantages:
\begin{inparaenum}[(1)]
\item the current model, dealing with the trace seen so far, is extended only conservatively, in that facts concerning the previous elements must be kept throughout the trace, and no new fact concerning those elements can be added, but
\item in continuations of the trace, new elements may appear, along with additional facts involving these.
\end{inparaenum}
This is fully in par with the expected interpretation of settings like that of Ex.~\ref{ex:concert}, where the prices and referred concerts of tickets seen in the past are fixed, while new information on unseen tickets can be acquired later.

\citeauthor{LeuS09} (\citeyear{LeuS09}) also emphasize another central monitoring desideratum, namely the identification of violations in the earliest possible instant, which in turn calls for \emph{anticipatory monitors} \cite{BartocciFFR18}. These monitors return their verdict not only considering the prefix seen so far, but also by reasoning on all its (infinitely many) possible suffixes. The next example makes a clear case for this.

\begin{example}
\label{ex:concert-monitor}
Consider formula $\psi$ from Ex.~\ref{ex:concert}.
Monitoring $\psi$ essentially amounts to check, at every instant, whether the expected relationship between $\bvar$ and $\tvar$ holds. Consider now
\[
\psi_{\mathtt{t123}} := \psi \land \F\G(\bvar = \mathtt{t123} 
)
\]
It expresses that at some point the agent selection stabilizes\footnote{This is captured by the \ltlf operators $\F$ and $\G$, where $\F\G\gamma$ models that there exists an instant from which $\gamma$ always holds; it is known that, over finite traces, this is logically equivalent to $\gamma$ holding in the last instant of the trace \cite{DeGV13}.} 
to the ticket identified by $\mathtt{t123}$. 
A monitor checking only the trace prefix seen so far would detect a violation of $\psi_{\mathtt{t123}}$ only if $\psi$ is violated, or at the end of the trace if the selected ticket is not $\mathtt{t123}$.
However, note that the price of ticket $\mathtt{t123}$ must be known, say $price(\mathtt{t123})=100$, and due to the behaviour of the agent, it is clear that $\mathtt{t123}$ will never be selected as soon as the agent finds a ticket with a price below 100 euros. In this respect, an anticipatory monitor for $\psi_{\mathtt{t123}}$ would immediately recognize  that $\psi_{\mathtt{t123}}$ is permanently violated (i.e., it will stay so independently on how the trace evolves) if in the current instant
$\bvar$ points to a ticket with a price $<100$. 
\end{example}

In this paper, we tackle for the first time anticipatory monitoring described in Ex.~\ref{ex:concert-monitor}. In fact, \emph{anticipatory monitoring for \LTLfMT is a completely open problem}. It is indeed a challenging one, as already in the case where \ltlf is equipped with numeric datatypes with simple forms of arithmetic, it incurs in undecidability~\cite{FMPW23}, and therefore relevant monitorable fragments of the logic have to be singled out. We bring a fourfold contribution:
\begin{inparaenum}[(1)]
\item We formally introduce the anticipatory monitoring problem for \LTLfMT specifications.
\item We focus on the quantifier-free fragment of \LTLfMT, an expressive logic called \qfLTLfMT that has been previously studied for process analysis \cite{GianolaMW25KI} and verification \cite{GianolaMW24} tasks, but never for monitoring. We provide a monitoring technique that combines an automata-theoretic component to deal with the temporal dimension, with an SMT component to tackle the data part.
\item We identify expressive classes of properties combining arithmetic with references to a database, in the style of Ex.~\ref{ex:concert}, where we show that our monitoring task becomes effectively solvable. 
\item
Feasibility of our method is witnessed by a prototype implementation with preliminary evaluation.
\end{inparaenum}

\paragraph{Related work}

A rich body of work studies RV of arithmetic temporal properties, notably Lola~\cite{DAngeloSSRFSMM05} (which also supports anticipation~\cite{HiplerKLS24}), TeSSLa~\cite{KallwiesLSSTW22}, and Eagle~\cite{GoldbergH05}. From \cite{FMPW23}, which considers anticipatory monitoring for \LTLf with arithmetic. We borrow in the present work the idea to abstract variable configurations by sets of formulae arranged in a graph, and the resulting approach to identify decidable fragments. However, dealing with richer theories substantially changes both the setting and the techniques needed to construct monitors. 

\citeauthor{DeLT16} (\citeyear{DeLT16}) propose an approach for monitoring modulo theories, but use an incomparable logic where variables cannot be compared across states, and models change at each trace instant. In Ex.~\ref{ex:concert}, this would essentially mean that the price and referred concert of already seen tickets non-deterministically change at every instant.
Also, no decidable fragments are identified. 
%
In another line of work~\cite{SchneiderBBKT21,BasinKMZ15,LimaMT24}, monitoring with \emph{metric} first-order temporal logic is considered, but in the non-anticipatory setting. 
\citeauthor{HavelundPU20} (\citeyear{HavelundPU20}) present a monitoring approach for first-order past time temporal logic without anticipation. This work shares with the present paper that infinitely many assignments are represented by logical formulas, in this case by BDDs. However, in both \cite{HavelundPU20,BasinKMZ15} the model is supposed to change at every instant, as predicates model events rather than facts.

The logic \qfLTLfMT we consider here has previously been studied for process analysis and verification. Specifically, \cite{GianolaMW24} verifies data-aware dynamic systems against \qfLTLfMT properties, using automata-based techniques that are however too weak to be ported to monitoring.

\section{Preliminaries}
\label{sec:preliminaries}

\paragraph{\LTLfMT syntax.}
We consider a first-order multi-sorted \emph{signature} $\Sigma=\langle\SS, \PP, \FF, V, U\rangle$, where 
$\SS$ is a set of sorts;
$\PP$ is a set of predicate and $\FF$ a set of function symbols over $\SS$;
$V$ is a set of \emph{data variables}; and
$U$ is a set of variables disjoint from $V$ that will be used for quantification; where all variables have a sort in $\SS$.
below, we consider a fixed signature $\Sigma$ that contains equality predicates for all sorts in $\SS$.

A term $t$ is generated by the following grammar:
\[
t := x \mid y \mid f (t_1, \dots , t_k) \mid \prev x
\]
where $x \in V$ and $y \in U$ are variables, $f \in \FF$ is a function symbol of arity $k$, each $t_i$, $1 \leq i \leq k$, is a term, and $\prev{}$ is the \emph{previous} constructor. 
The grammar of \LTLfMT formulas over $\Sigma$ is the following:
\begin{align*}
\alpha &:= p(t_1, \dots, t_k) \\
\phi &:= \alpha \mid \neg \alpha \mid \phi \vee \phi' \mid \phi \wedge \phi' \mid \exists y.\phi \mid \forall y.\phi \\
\psi &:= \top \mid \phi \mid \psi\,{\vee}\,\psi' \mid \psi \,{\wedge}\,\psi' \mid \X \psi \mid \wX \psi \mid \psi\,{\U}\,\psi' \mid \psi\,{\R}\,\psi'
\end{align*}
%
where $y \in U$, $p \in \PP$ is an $k$-ary predicate symbol, each
$t_i$ is a term, and $\X$, $\wX$, $\U$, and $\R$ are the next, weak next, until, and release temporal operators, respectively.
Formulas of type $\phi$ as defined above are called first-order (FO) formulas.
In the definition of $\psi$, we force $\phi$ to have no free variables from $U$.  We use the usual shortcuts $\F \psi \equiv \top \U \psi$ and $\G \psi \equiv \bot \R \psi$. 
By the grammar above, \LTLfMT formulas are always in negation
normal form, and we assume them to be well-typed with regard to the sorts of all the involved symbols.
The set of all \LTLfMT formulas is denoted $\LL$.

For instance, $\psi$ and $\psi_{\mathtt{t123}}$ in Ex.~\ref{ex:concert} are \LTLfMT formulas over the signature where $V = \{\tvar, \bvar\}$, there are no quantified variables so $U$ can be chosen empty, $\PP$ contains the arithmetic predicates, and $\FF$ consists of  arithmetic functions, equality, $\mathit{price}$, $\mathit{concert}$, $\mathtt{undef}$, $\mathtt{myc}$, and $\mathtt{t123}$. 

A FO formula $\phi$ is a \emph{state formula} if it does neither contain free variables from $U$ nor $\prev{}$ operators.
A state formula without free variables is a \emph{sentence}, and a
set of sentences is a 
\emph{$\Sigma$-theory} $\TT$. It is \emph{universal} if all its sentences are over $\Sigma$ and have the form $\forall u_1, \dots, u_l.\, \psi$ with $\psi$ quantifier-free.

\paragraph{\LTLfMT semantics.}
To define semantics, we use the standard notion of a \emph{$\Sigma$-structure} $M$,
which associates each sort $s\in \SS$ with a domain $s^M$, and each predicate $p\in \PP$ and function symbol $f\in \FF$ with a suitable interpretation $p^M$ and $f^M$. The equality predicates have the standard interpretation given by the identity relation.
The carrier of $M$, i.e., the union of all domains of sorts in $\SS$, is denoted by $\Mcarrier$.
%
A total function $\alpha\colon V \to \Mcarrier$ is called a \emph{state variable assignment}.
We always assume that variables are mapped to an element of their respective domain. 
A \emph{trace} $\tau=\trace$ consists of a model $M$ and a sequence of state variable assignments $\alpha_i$ with codomain $\Mcarrier$.
We call $n$ the \emph{length} of $\tau$, denoted $|\tau|$.
For two traces $\tau=\trace[n]$ and $\tau'=(M',\langle\alpha_0', \dots,\alpha_m'\rangle)$ with $M=M'$, their concatenation is given by $\tau \cdot \tau':=(M,\langle\alpha_0, \dots, \alpha_n,\alpha_{0}', \dots, \alpha_m'\rangle)$.

Given a trace $\tau=\trace$, 
and a variable evaluation function
$\xi \colon U \to \Mcarrier$ for the variables in $U$,
the \emph{evaluation} $\eval[i]{t}$ of a term $t$ is 
\emph{well-defined} if $i < |\tau|$, and either $i>0$ or $t$ does not contain $\prev{}$ operators, namely as\\
$\begin{array}{@{}r@{}lr@{}l@{}}
\eval{v} &= \alpha_i(v)&
\eval{\prev v} &= \alpha_{i-1}(v) \\
\eval{u} &= \xi(u) &
\eval{f(t_1,\dots, t_k)} &= f^{M}(\eval{t_1}, \dots, \eval{t_k})
\end{array}$

where $v \in V$ and $u\in U$.
Satisfaction of a first-order formula $\phi$ with respect to an environment $\xi$ in the run $\tau$ with $i < |\tau|$, denoted 
$\tau \models^i_\xi \phi$, is defined as:
\begin{alignat*}{3}
\tau \models^i_\xi & p(t_1, \dots, t_k) && \text{if $t_1, \dots, t_k$ are well-defined and} \\[-1ex]
&&&(\eval{t_1}, \dots ,\eval{t_k}) \in p^{M_i},\\[-0.5ex]
\tau \models^i_\xi & \neg p(t_1, \dots, t_k) \quad &&
\text{if }\tau \not\models^i_\xi p(t_1, \dots, t_k)\\
\tau \models^i_\xi & \phi_1 \land \phi_2 &&
\text{if }\tau \models^i_\xi \phi_1 \text{ and }\tau \models^i_\xi \phi_2\\
\tau \models^i_\xi & \phi_1 \lor \phi_2 &&
\text{if }\tau \models^i_\xi \phi_1 \text{ or }\tau \models^i_\xi \phi_2\\
\tau \models^i_\xi & \exists y.\,\phi &&
\text{if }\tau \models^i_{\xi[y\mapsto e]} \phi
\text{ for some } e\in \Mcarrier\\
\tau \models^i_\xi & \forall y.\,\phi &&
\text{if }\tau \models^i_{\xi[y\mapsto e]} \phi
\text{ for all } e\in \Mcarrier
\end{alignat*}
where $y$ is assumed to have sort $s$, and $e$ is in the domain for $s$ in $\Mcarrier$. 
Satisfaction w.r.t.~$\tau$ is extended to a general \LTLfMT formula $\psi$ as follows:
\[
\begin{array}{@{}r@{\,}ll@{}}
\tau \models^i & \phi &
\text{if } \tau \models^i_\emptyset \phi \\
\tau \models^i & \psi_1 \land \psi_2 &
\text{if }\tau \models^i \psi_1 \text{ and }\tau \models^i \psi_2\\
\tau \models^i & \psi_1 \lor \psi_2 &
\text{if }\tau \models^i \psi_1 \text{ or }\tau \models^i \psi_2\\
\tau \models^i & \X \psi &
\text{if $i<|\tau|{-}1$ and }\tau \models^{i+1} \psi \\
\tau \models^i & \wX \psi &
\text{if $i\geq|\tau|{-}1$ or }\tau \models^{i+1} \psi \\
\tau \models^i & \psi_1 \U \psi_2 &
\text{if there is some $j$, $i \leq j<|\tau|$ such that } \\
&& \tau \models^{j} \psi_2\text{ and }\tau \models^{k} \psi_1 \text{ for all }i \leq k <j \\
\tau \models^i & \psi_1 \R \psi_2 &
\text{if either }\tau \models^{j} \psi_2 \text{ for all }i \leq j<|\tau|\text{, or there is}\\
&&\text{some $j$, $i \leq j<|\tau|$ such that }\tau \models^{j} \psi_1 \\
&& \text{and }\tau \models^{k} \psi_2 \text{ for all }i \leq k \leq j.
\end{array}
\]
Finally, $\tau$ \emph{satisfies} $\psi$, denoted by $\tau \models \psi$,  if
$\tau \models^0 \psi$ holds.

\newcommand{\cnull}{\mathit{null}}
\begin{example}
\label{ex:concert:2}
For $\phi$ from Ex.~\ref{ex:concert},
consider a model $M$ where $\mathit{Concert}^M=\{c_0, c_1, c_2, c_3\}$,
 $\mathtt{myc}^M=c_1$, 
 $\mathit{Ticket}$ has carrier $\{\bot, t_1, t_2, \dots,t_{1000}\}$,
 $\mathtt{undef}^M =\bot$, 
 with an interpretation that satisfies 
$\cfun(t_2) = \cfun(t_4) = \cfun(t_5) = c_1$ and $\cfun(t_1) = c_2$, $\pfun(t_2)=100$, and $\pfun(t_4)=110$, $\pfun(t_5)=95$. 
Then the trace $(M, \overline \alpha)$ satisfies $\phi$, where
\[
\overline \alpha = \langle 
\domino{\bot}{\bot},
\domino{t_2}{100},
\domino{t_1}{100},
\domino{t_4}{100},
\domino{t_5}{95}
\rangle
\]
\end{example}

Our logic differs from~\cite{GGG22} in two respects: we define \LTLfMT with a \emph{lookback} operator $\prev{}$ instead of \emph{lookahead}, i.e., the logic allows to refer to values of $V$ in the previous trace instant, not in the next. However, every formula with lookahead can be equivalently expressed with lookback~\cite[Lem. 37]{FMPW23}.
Moreover, we have the strong version of the cross-state operator (but not its weaker version), so that atoms with not well-defined terms are false. 

\paragraph{First-order logic.}
Next, we need to repeat some preliminaries from FO logic and model theory.
If $\phi$ is a state formula, we simply write $M,\alpha \models \phi$ instead of $\langle(M,\alpha)\rangle \models_\emptyset^i \phi$, and
if $\phi$ is a sentence, we write
$M \models \phi$ for
$M,\emptyset \models_\emptyset^0 \phi$.
We adopt a common SMT perspective and view a $\Sigma$-theory $\TT$ as the set of all $\Sigma$-structures that satisfy all axioms of $\TT$, and write
$M \in \TT$ if $M$ is a model of $\TT$, i.e.,
$M\models \phi$ holds for all sentences 
$\phi$ in $\TT$.
A state formula $\phi$ is \emph{$\TT$-satisfiable} if there is some $M\in \TT$ and state variable assignment $\alpha\colon V \to \Mcarrier$ such that $M,\alpha \models \phi$;
whereas $\phi$ is $\TT$-valid if $M,\alpha \models \phi$ for all $M\in \TT$ and $\alpha$.
Moreover, formulas $\phi_1$ and $\phi_2$ are \emph{$\TT$-equivalent}, denoted
$\phi_1 \equiv_\TT \phi_2$, if $\phi_1 \leftrightarrow \phi_2$ is $\TT$-valid.

A $\Sigma$-theory $\TT$ has \emph{quantifier elimination} (QE) if for every $\Sigma$-formula $\phi$
 there is a quantifier-free formula $\phi'$ that is $\TT$-equivalent to $\phi$.
As QE is a strong requirement,
we exploit in this paper the weaker notion of 
a theory having a model completion.
%
In this context, an \emph{embedding} is a mapping $\mu \colon M \to M'$ between two $\Sigma$-structures $M$ and $M'$ that preserves the truth of $\Sigma$-literals.
It is easy to see that a composition of embeddings is again an embedding.
%
A universal $\Sigma$-theory $\TT$ has model completion~\cite{Ghilardi04} iff there exists a $\Sigma$-theory $\TT^*$, i.e., a theory over the same signature, where $\TT^* \supseteq \TT$, such that
$(i)$ every model of $\TT$ can be embedded in a model of $\TT^*$, and
$(ii)$ $\TT^*$ has QE. 

It follows from this definition that every model of $\TT^*$ can be embedded in a model of $\TT$. 
We will later exploit the following properties \cite[Sec. 2]{CalvaneseGGMR20}:
\begin{lemma}\label{lem:embedding}
  Let $\TT$ have model completion $\TT^*$.
\begin{compactenum}[(1)]
  \item
    Let $M\in \TT$, $\varphi$ be an existential $\Sigma$-formula and $\alpha$ an assignment of values in $M$ such that $M,\alpha \models \varphi$. Then there is an embedding $\mu \colon M \to M'$ for some $M' \in \TT^*$ such that $M', \mu \circ \alpha \models \varphi$.
\item
    Let $M\in \TT$, $M' \in \TT^*$, $\mu\colon M \to M'$ be an embedding,
     $\varphi$ a quantifier-free $\Sigma$-formula, and $\alpha$ be an assignment of values in $M$.  
  If $M',\mu \circ \alpha \models \varphi$ then $M,\alpha \models \varphi$.
\end{compactenum}
\end{lemma}

For a trace $\tau = \trace$ and an embedding $\mu \colon M \to M'$, let
$\mu(\tau) = (M', \langle \mu \circ \alpha_0, \dots, \mu \circ \alpha_n \rangle)$.

We will sometimes refer to common SMT theories~\cite{BarrettSST21}: the theory of equality and uninterpreted functions  (EUF) for a given $\Sigma$, and linear arithmetic over rationals (LRA), integers (LIA), or both (LIRA). While the arithmetic theories have QE~\cite{Presburger29}, EUF admits model completion and so do certain 
\emph{tame} combinations of LIRA and EUF~\cite{IJCAR20,CalvaneseGGMR22}.
%
In the remainder, we make the following assumptions on the theory: 
\begin{assumptions}
\label{ass:T}
\begin{inparaenum}[(a)]
\item Satisfiability of existential $\TT$-formulas is decidable;
\item  $\TT$ is either universal and has a model completion $\TT^*$; or has QE, in this case let $\TT^*\,{:=}\,\TT$.
\end{inparaenum}
\end{assumptions}
We will show later that these assumptions are satisfied by several theories which are highly relevant in practice.

\paragraph{Monitoring.}
We now define our main task, namely determining how the satisfaction of a given property changes along a trace.
Throughout the paper, we consider \emph{anticipatory} monitoring, taking into account both the trace seen so far and all its finite future continuations.  
Following the literature~\cite{BaLS10}, we consider the set $RV = \{\PS, \CS, \CV, \PV\}$ of four distinct \emph{monitoring states}:
current satisfaction ($\CS$), permanent satisfaction ($\PS$), current violation ($\CV$) and permanent violation ($\PV$).
We call a trace $\tau'$ is an \emph{extension} of a trace $\tau=\trace$ if there is some $M'$ such that there is an embedding $\mu \colon M \to M'$ and assignments $\alpha_i'$ such that $\tau' = (M',\langle\mu(\alpha_0), \dots, \mu(\alpha_n),\alpha_{0}', \dots, \alpha_m'\rangle)$.

\begin{definition}
\label{def:monitoring}
An \LTLfMT property $\psi$ is in
monitoring state $s \in RV$ after a trace $\tau$, written $\tau \models \llbracket \psi = s\rrbracket$, if
\begin{compactitem}
\item
$s = \CS$, $\tau \models \psi$ but $\tau' \not\models \psi$ for some extension $\tau'$ of $\tau$;
\item
$s = \PS$, $\tau \models \psi$, and $\tau' \models \psi$ for every extension $\tau'$ of $\tau$;
\item
$s = \CV$, $\tau \not\models \psi$ but $\tau' \models \psi$ for some extension $\tau'$ of $\tau$;
\item
$s = \PV$, $\tau \not\models \psi$, and $\tau' \not\models \psi$ for every extension $\tau'$ of $\tau$.
\end{compactitem}
\end{definition}

After each trace, a property $\psi$ is in exactly one possible monitoring state.
E.g. for the trace $\tau$ from Ex.~\ref{ex:concert:2} and $\psi$, $\psi_{\mathtt{t123}}$ from Ex.~\ref{ex:concert} we have 
$\tau \models \llbracket \psi = \CS\rrbracket$: the property is currently satisfied but could be violated in an extension of $\tau$. Considering $\psi_{\mathtt{t123}}$, let $\mathtt{t123}^M=t_{123}$. If we have $\pfun(t_{123}) = 80$ then
$\tau \models \llbracket \psi_{\mathtt{t123}} = \CV\rrbracket$ as $t_{123}$ could still be selected in the future; but if 
$\pfun(t_{123}) = 100$ then $\tau \models \llbracket \psi_{\mathtt{t123}} = \PV\rrbracket$.

Given $\tau$ and $\psi$,
the \emph{monitoring problem}
is to compute the state $s\in RV$ s.t.  
$\tau \models \llbracket \psi = s\rrbracket$. It is \emph{solvable} if
one can construct a procedure for $\psi$ that computes the monitoring
state for any given trace.
It is known that the monitoring problem is not solvable if $\TT$ is the theory of linear arithmetic over the reals~\cite{FMPW23}, and hence also not for the much more general logic \LTLfMT.

\section{Automata for Properties}
\label{sec:automata}

In the remainder of the paper we assume a theory $\TT$ that satisfies Assumptions~\ref{ass:T}, and the quantifier-free fragment of \LTLfMT, that is, the logic \qfLTLfMT from \cite{GianolaMW25KI,GianolaMW24}. Excluding quantifiers is here an essential restriction when applying reasoning with model completion to the automaton, which is needed to construct the monitors. As e.g.~shown in Ex.~\ref{ex:concert}, \qfLTLfMT still supports expressive queries over the theories, provided that they are formulated referring to the monitored variables. 

It is well-known that \LTLf properties can be captured by non-deterministic automata (NFA)~\cite{dGV13}. Our monitoring approach relies on such an NFA for a given \LTLfMT property $\psi$. While the construction is rather standard, the correctness proofs are more complicated due to the FO setting.
We only sketch here our construction, details and proofs can be found in the appendix.
Let $\lambda$ be an additional proposition to mark the last element of a trace.
Given a \qfLTLfMT property $\psi$, let $C$ be the set of FO formulas in $\psi$ and $C^\pm = C \cup\{ \neg c \mid c\inn C\}$ the set of FO formulas in $\psi$ together with their negations.
The alphabet of the NFA will be $\smash{\Theta=2^{C^\pm \cup \{\lambda,\neg\lambda\}}}$, so that a symbol is a set of FO formulas. 

We build an NFA $\NFA[\psi]$ using an auxiliary function $\delta$:
The input of $\delta$ is a property $\psi \in \LL\cup \{\top,\bot\}$.
The output of $\delta$ is a set of pairs
$(\inquotes{\psi'},\varsigma)$ where $\psi' \in \LL\cup \{\top,\bot\}$ is again
a property and $\varsigma \in \Theta$. 
%
%
Intuitively, $(\inquotes{\psi'},\varsigma)\in\delta(\psi)$ expresses that 
a word $\varsigma w$ satisfies $\psi$ iff the suffix $w$ satisfies $\psi'$, for its definition see the appendix.
%
%
%
The NFA $\NFA$ is then defined as follows:
\begin{definition}
\label{def:NFA}
Given a property $\psi \in \LL$, let
$\NFA\,{=}\,(Q, \Theta, \varrho, q_0, Q_F)$ where $q_0\,{=}\,\inquotes{\psi}$ is the initial state,
$Q_F = \{\inquotes{\top}, q_{+}\}$ is the subset of final states, 
 $q_{-}$ is an additional non-final state,
and
$Q$, $\varrho$ are the smallest sets such that $q_0, q_F, q_{+}, q_{-} \in Q$ and whenever 
$q\in Q\setminus\{q_{+}, q_{-}\}$ and $(q', \varsigma)\in \delta(q)$
then $q'\in Q$ and 
\begin{compactenum}[(i)]
\item if $\lambda \not\in \varsigma$ then
$(q, \varsigma, q') \in \varrho$, and
\item
whenever $\lambda \in \varsigma$, if $q' = \inquotes{\top}$ then
$(q, \varsigma, q_{+}) \in \varrho$, and if
$q' = \inquotes{\bot}$ then
$(q, \varsigma \setminus\{\lambda\}, q_{-}) \in \varrho$.
\end{compactenum}
\end{definition}

\begin{example}
Consider $\psi$ from Ex.~\ref{ex:concert}. 
We abbreviate
$A:= (b=\undef)$,
$B:= (\cfun(t) = \myc)$,
$C:= (\pfun(t) > \pfun(\prev b))$,
$D:= (b=t)$,
$E:= (b=\prev b)$, and 
$\phi' := \G (\phi_{change} \wedge \phi_{keep})$.
Then the NFA $\NFA$ is as follows:

\noindent
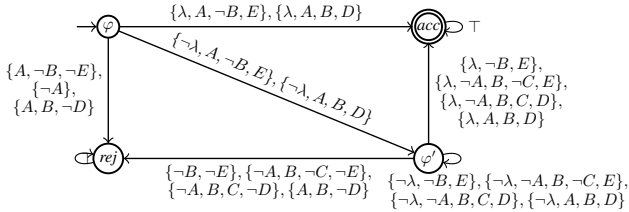
\resizebox{\columnwidth}{!}{
\begin{tikzpicture}[node distance=68mm]
 \node[state] (phi) {$\phi$};
 \node[state, right of=phi,  final] (acc) {acc};
 \node[state, below of=phi,yshift=40mm,] (rej) {rej};
 \node[state, right of =rej] (phi2) {$\phi'$};
\draw[edge] ($(phi) + (-.4,0)$) -- (phi);
\draw[edge] (phi) -- node[action, above] {$\{\lambda, A, \neg B, E\}, \{\lambda, A, B, D\}$} (acc);
\draw[edge] (phi) -- node[action, above, sloped] {$\{\neg\lambda, A, \neg B, E\}, \{\neg \lambda, A, B, D\}$} (phi2);
\draw[->] (phi2) to[loop right, looseness=8] node[action, below, xshift=10mm, yshift=-2mm] {$\begin{array}{@{}c@{}}\{\neg\lambda, \neg B, E\}, \{\neg \lambda, \neg A, B, \neg C, E\},\\\{\neg \lambda, \neg A, B, C, D\},\{\neg \lambda, A, B, D\}\end{array}$} (phi2);
\draw[edge] (phi2) -- node[action, right] {$\begin{array}{@{}c@{}}\{\lambda, \neg B, E\}, \\\{\lambda, \neg A, B, \neg C, E\},\\\{\lambda, \neg A, B, C, D\},\\\{\lambda, A, B, D\}\end{array}$} (acc);
\draw[edge] (phi2) -- node[action, below] {$\begin{array}{@{}c@{}}\{\neg B, \neg E\}, \{\neg A, B, \neg C, \neg E\},\\\{\neg A, B, C, \neg D\},\{A, B, \neg D\}\end{array}$} (rej);
\draw[edge] (phi) -- node[action, left] {$\begin{array}{@{}c@{}}\{A, \neg B, \neg E\}, \\\{\neg A\},\\\{A, B, \neg D\}\end{array}$} (rej);
\draw[->] (acc) to[loop right, looseness=8] node[action, right] {$\top$}  (acc);
\draw[->] (rej) to[loop left, looseness=8] node[action, right] {$\top$}  (rej);
\end{tikzpicture}
}
\end{example}

We denote by $\Vpre$ the set of variables
$\Vpre{=}\{\prev v \mid v\inn V\}$.
A word $w = \varsigma_0, \varsigma_1, \dots, \varsigma_{n-1}\in \Theta^+$ is \emph{well-formed} if 
$\varsigma_0$ contains no variable in $\Vpre$.
We next define the notion of \emph{consistency}, which serves to capture that a trace $\tau$ satisfies all formulas in a word $w\inn\Theta^+$ as above. 
To that end, for assignments $\alpha$ and $\alpha'$ with domain $V$, 
we define 
the assignment $\combine{\alpha}{\alpha'}$ with domain $\Vpre \cup V$ as 
$\combine{\alpha}{\alpha'}(\prev {v}) = \alpha(v)$ and 
$\combine{\alpha}{\alpha'}(v) = \alpha'(v)$, for all $v\in V$.
For simplicity, we write $M, \alpha \models \varsigma$ instead of $M, \alpha \models \bigwedge \varsigma$.

\begin{definition}
A well-formed word $\varsigma_0, \varsigma_1,\dots, \varsigma_{n-1} \in \Theta^+$ is \emph{consistent} with a trace $\trace$ if $M,\alpha_0 \models \varsigma_0$ and
$M,\combine{\alpha_{i-1}}{\alpha_{i}} \models \varsigma_i$
for all $0 < i < n$, and moreover $\lambda \in \varsigma_{n-1}$, and $\neg \lambda \in \varsigma_i$ for all $0 \leq i < n-1$.
\end{definition}

\noindent
To state the main result about the NFA $\NFA$ below, we need to restrict to properties that use lookback in a valid way, i.e., that do not enforce constraints with lookback variables in the first trace instant.
To that end, let an \qfLTLfMT property $\psi$ have \emph{safe lookback} if no edge from the initial state has a label that contains $\Vpre$. 
The following is the main result about $\NFA$:

\begin{restatable}{theorem}{thmNFA}
\label{thm:automaton:acceptance}
Given $\psi\inn\LL$ with safe lookback and trace $\tau$, 
\begin{inparaenum}[(1)]
\item
if $\mathcal{N}_{\psi}$ accepts a well-formed word $w\inn \Theta^+$ consistent with $\tau$, then $\tau \models \psi$, and
\item if $\tau \models \psi$ then there exists a well-formed word $w\inn \Theta^+$ consistent with $\tau$ and accepted by $\mathcal{N}_{\psi}$.
\end{inparaenum}
\end{restatable}

\noindent
For $w = \varsigma_0, \dots, \varsigma_{n-1}$, we write $q_0 \goto{w} q_{n}$ if there is a sequence of states
$q_1, \dots, q_{n-1}$ such that $(q_{i}, \varsigma_i, q_{i+1}) \in \varrho$ for all $0 \leq i < n$.
We have the following determinism property:

\begin{lemma}
\label{lem:word:for:trace}
Given a trace $\tau$ there are a unique word $w\inn \Theta^*$ and state $q$ s.t.
$q_0 \goto{w} q$ in $\NFA$ and
$w$ is consistent with $\tau$.
\end{lemma}

Note that while it is possible to postprocess $\NFA$ to an NFA for $\psi$ that does not mention $\lambda$~\cite{FMPW23}, the resulting automaton requires an explicit determinization step to obtain a property akin to Lem.~\ref{lem:word:for:trace}, so we prefer to keep the $\lambda$'s.


\section{Monitoring Technique}

We next present our monitoring algorithm. 
Throughout this section, we assume a \qfLTLfMT property $\psi$ with NFA $\NFA\,{=}\,(Q, \Theta, \varrho, q_0, Q_F)$ as defined in the previous section.

For a formula $\varphi$ with free variables $V$, let $\varphi(\vec V/\vec X)$ denote the formula
where $\vec V$ is replaced by $\vec X$ (we denote by $\vec V$ a fixed ordering of $V$ as a vector).
For a set of constraints $C$ with free variables $\Vpre\cup V$, 
let $C(\vec X/\vec \Vpre, \vec Y/\vec V)$
denote the constraints where $\vec \Vpre$ is replaced by $\vec X$ and $\vec V$ by $\vec Y$.
Moreover, for a set of constraints $C \in \Theta$, we denote by $\nolambda C$ its conjunction where $\lambda$s are omitted, i.e., $\nolambda C = \bigwedge C \setminus \{\lambda, \neg \lambda\}$; so $\nolambda C$ is a formula with free variables $\Vpre \cup V$.

\begin{definition}
\label{def:regress}
For a state formula $\varphi$
and $C \inn \Theta$,
let
$\regress(\varphi, C) =$ $\exists \vec Y. \nolambda C(\vec V/\vec \Vpre, \vec Y/\vec V)\wedge \varphi(\vec Y/\vec V)$, where $Y$ is a set of fresh variables of the same size and sorts as $V$.
\end{definition}

Note that since $\varphi$ has free variables $V$, also $\regress(\varphi, C)$ has free variables $V$.
The $\regress$ operation defined above allows us to ``propagate backwards'' the verification obligation expressed by $\varphi$ through the constraint $C$: in formal verification, this corresponds to the classic notion of \emph{preimage} when performing backward search in a transition system, cf. e.g. \cite{AbdullaCJT96}. 
We collect these verification obligations in a structure called coreachability graph, which is one of the technical novelties of this paper, defined next:

\begin{definition}\label{def:cg}
  Given NFA $\NFA$ for $\psi$,
the \emph{coreachability graph} $\CG^+(\NFA)$ for $\NFA$ is a tuple
$\langle S, \gamma\rangle$ where $S$ is a set of nodes of the form $(q,\varphi)$ for $q\in Q$ and $\varphi$ a quantifier-free $\Sigma$-formula with free variables $V$, and $\gamma \subseteq S \times \Theta \times S$ is a transition relation.
It is inductively defined as follows:
\begin{compactitem}
\item[$(i)$] for all final states $q\in Q_F$, a node $(q,\top)\in S$, and
\item[$(ii)$] if $s  \in S$ is of the form $s= (q,\varphi)$ and there is a transition $q' \goto{\varsigma} q$ in $\NFA$ such that
$\regress(\varphi, \varsigma)$ is $\TT$-satisfiable,
there is some node $s' = (q', \varphi')\in S$
such that $\varphi' \sim \regress(\varphi, \varsigma)$, and 
$s' \goto{\varsigma} s$ is in $\gamma$.
\end{compactitem}
\end{definition}

Note that we use model completion to obtain quantifier-free formulas, which is the main technical novelty of our approach; 
the fact that all formulas in CGs are quantifier-free allows us to identify decidable fragments in the next section. Note that as $\regress(\varphi, \varsigma)$ is existential, by Assumptions~\ref{ass:T} some quantifier-free $\chi$ s.t. $\chi \equiv_{\TT^*} \regress(\varphi, \varsigma)$ indeed exists, by applying quantifier elimination in $\TT^*$. 

Intuitively, if $\CG^+(\NFA)$ has a node $(q,\varphi)$ then this means that from NFA state $q$ a final state can be reached under the condition that the current variable assignment satisfies $\varphi$. It is called '\emph{co}reachability graph' because it is is constructed by \emph{regressing} from the final states of the system, rather than by \emph{forward exploration} from the initial states as customary in standard reachability graphs \cite{DeselE93,ChiolaDFH97}. 
The \emph{non-coreachability graph} $\CG^-(\NFA)$ is defined dually, but where in condition $(i)$ of Def.~\ref{def:cg} we have $(q,\top)\in S$ for all $q\in Q \setminus Q_F$.

The construction in Def.~\ref{def:cg} need not terminate as infinitely many formulas may occur, but we will show in the next section that termination is guaranteed for  several relevant classes of properties. 
For cases where $\CG^+(\NFA)$ and $\CG^-(\NFA)$ are finite, we define the following formulas which express, intuitively, under which conditions a final (resp. non-final) state is reachable from state $q$ in the NFA.
\begin{definition}
\label{def:satunsat}
For $\psi$ with NFA $\NFA$ and a state $q$ in $\NFA$, let\\[.5ex]
$\begin{array}{r@{\,}l}
\sat(q) =&  \bigvee \{\varphi \mid (q, \varphi)\text{ is a node in }\CG^+(\NFA)\}
\\
\unsat(q) =&  \bigvee \{\varphi \mid (q, \varphi)\text{ is a node in }\CG^-(\NFA)\}
\end{array}$
\end{definition}

This leads to the following monitoring procedure, where the trace $\tau$ is assumed to have positive length:
\begin{algorithmic}[1]
\Procedure{monitor}{$\psi$, $\tau$}
  \State compute $\NFA[\psi]$
  \State $q \gets$ state in $\NFA[\psi]$ such that $q_0\gotos{w} q$ for some well-
  
  formed $w\in \Theta^*$ consistent with $\tau$ \Comment{cf. Lem.~\ref{lem:word:for:trace}}
  \State let $M$ be the model of $\tau$ and $\alpha_n$ its last assignment
  \If{$q$ is accepting in $\NFA[\psi]$}
    \State \textbf{return} $\CS$ \textbf{if} $M,\alpha_n \models\unsat(q)$ \textbf{else} $\PS$
  \EndIf
  \State \textbf{else return} $\CV$ \textbf{if} $M,\alpha_n \models\sat(q)$ \textbf{else} $\PV$
\EndProcedure
\end{algorithmic}

Thanks to Thm.~\ref{thm:automaton:acceptance}, Line 5 is equivalent to checking whether the property $\psi$ holds on the prefix $\tau$. If so, the possible monitoring states are restricted to $\CS$ and $\PS$; therefore, in Line~7 it suffices to check whether the extension leads to a violation (i.e., $\unsat(q)$). 
If $\tau$ does not satisfy $\psi$,  a similar reasoning applies.

The NFA, CGs, and formulas $\sat$ and $\unsat$ should only be computed when required but can be cached for later use, to avoid re-computation. This is especially useful when monitoring multiple traces w.r.t. the same property.

\paragraph{Correctness.}
To prove correctness,
we need some additional notation. First of all, \emph{regression constraints}
are formulas obtained from nested $\regress$ operations:

\begin{definition}
\label{def:regression}
For $w\,{=}\,\varsigma_0,\dots, \varsigma_{n-1}$ in $\Theta^*$,
the \emph{regression constraint} $\precond(w)$ is given by
$\precond(w) = \top$ if $n= 0$, and if $n\,{>}\,0$ then
$\precond(w) = \regress(\precond(\langle\varsigma_1,\dots, \varsigma_{n-1}\rangle), \varsigma_{0})$.
\end{definition}

Crucial for our verification technique is the fact that there is a correspondence between satisfying assignments of $\precond(w)$, and traces that satisfy the verification obligations
given by the constraints in $w$, as stated next:

\begin{restatable}{lemma}{lemmaregression}
\label{lem:regression}
For $w\inn \Theta^*$, a model $M$ and assignment $\alpha$,
\begin{inparaenum}[(1)]
\item if $M,\alpha\,{\models}\,\precond(w)$
there is a trace $\tau$ with model $M$ s.t.
$\langle\emptyset\rangle\,{\cdot}\,w$ is consistent with $(M,\langle\alpha\rangle)\,{\cdot}\,\tau$, and 
\item  
if $\langle\emptyset\rangle\,{\cdot}\,w$ is consistent with some $(M,\langle\alpha\rangle)\,{\cdot}\,\tau$ then $M,\alpha \models \precond(w)$.
\end{inparaenum}
\end{restatable}

There is a close relationship between paths in CGs and regression constraints, as stated in the next lemma.
For a word $w\,{=}\, \varsigma_1, \dots, \varsigma_n$, we write $(q, \varphi) \gotos{w} (q', \top)$
if there is a path
$(q, \varphi) \goto{\varsigma_1} \dots \dots \goto{\varsigma_n} (q', \top)$
in $\CG^+(\NFA)$ or $\CG^-(\NFA)$.

\begin{restatable}{lemma}{lemmacg}
\label{lem:cg}
Let $G$ be $\CG^+(\NFA)$ (resp. $\CG^-(\NFA)$).
\begin{compactenum}[(1)]
\item If $G$ has a path $(q, \varphi) \gotos{w} (q_f, \top)$ then
$q \gotos{w} q_f$ in $\NFA$ such that $q_f$ is a final (resp. non-final) state, $\varphi$ is $\TT$-satisfiable, and $ \varphi\sim \precond(w)$.
\item If $q \gotos{w} q_f$ in $\NFA$ such that $q_f$ is final (resp. non-final), and $\precond(w)$ is $\TT$-satisfiable, there is a path
$(q, \varphi) \gotos{w} (q_f, \top)$ in $G$ for some $\varphi$ such that $\varphi \sim \precond(w)$.
\end{compactenum}
\end{restatable}

\noindent
At this point we are ready to prove our main theorem, which strongly exploits the two properties of model completions: embeddability of $\TT$- models into $\TT^*$-models, and 
QE in $\TT^*$. The proof of the following lemma is in the appendix.

\begin{theorem}
\label{thm:soundness}
If \textsc{monitor}$(\psi, \tau)$ returns $s$ then $\tau\,{\models}\,\llbracket \psi{=}s \rrbracket$.
\end{theorem}

\section{Solvability Criteria}
\label{sec:decidability}

In this section, we identify classes of \qfLTLfMT properties where \textsc{monitor} is guaranteed to terminate; we use the terminology that
the monitoring task is  \emph{solvable} in this case.
We focus on properties that combine references to a database with arithmetic, a combination that is very common in practice, in particular in data-aware BPM \cite{GGMR23} and when observing database-driven systems \cite{BojanczykST13,CGM13,DHLV18};

To define databases in a symbolical context, we rely on the notion of \emph{static DB schemas} from \cite{GGMR23}, where  relational databases with integer and real values are formalized in an algebraic way. Static DB schemas use the combination of two FO theories, $\TT_{db}$ and $\TT_{ar}$.
The theory  $\TT_{db}$ models read-only DB relations with primary and foreign key constraints by employing function symbols to represent the key dependencies: formally, this can be captured in SMT via constants and functions subject to multi-sorted EUF, as done in our ticketing Ex.~\ref{ex:concert}. 
The theory $\TT_{ar}$ is LIA, LRA or their combination; it is used to formalize arithmetic constraints.

\paragraph{DB-driven theories}
Let a \emph{\dataLTLf property} be a property in \LTLfMT where the signature $\Sigma$ consists of arbitrary relation, and unary function symbols. The associated theory $\TT_{db}$ is defined as EUF, possibly extended with a set of ground facts to model an initial DB. This theory enjoys model completion~\cite{CalvaneseGGMR20}.
In the case with arithmetic, $\Sigma$ supports in addition operations from an arithmetic theory $\TT_{ar}$ as above,
and its theory is $\TT:=\TT_{db}\cup \TT_{ar}$.
To ensure that $\TT$ has model completion, we assume that the signature is \emph{tame},
which intuitively means that no uninterpreted function symbol has an arithmetic domain (see below).

Note that the presence of only finitely many terms guarantees that the procedure \textsc{monitor} always terminates. However, in general both the DB and the arithmetic perspective can give rise to infinitely many terms, so in classes of properties where the monitoring task is solvable, both must be suitably restricted. Next, we explore different ways to do so.

\paragraph{I. Acyclic signature.}
First, we consider {\dataLTLf} without arithmetic.
Acyclicity of the signature $\Sigma$ is defined via the \emph{sort graph} of $\Sigma$, which
has as node set the sorts $\SS$, and an edge from $s$ to $s'$ iff there is a function symbol $f\colon s \to s'$ in $\Sigma$ (as we have only unary functions); 
we say that $\Sigma$ is \emph{acyclic} if the sort graph is so. 
For instance, the non-arithmetic part of the signature of Ex.~\ref{ex:concert} is acyclic.
In this case, only finitely many $\Sigma$-terms exist~\cite{ARS2010}. Thus, there are also only finitely many non-equivalent quantifier-free formulas over the finite set of variables $V$ \cite{CalvaneseGGMR20}.
and hence there are only finitely many regression constraints, so that the coreachability graph construction in Def.~\ref{def:cg} terminates.
Hence, we obtain the following decidability result:

\begin{theorem}
\label{thm:acyclic}
If $\Sigma$ is acyclic, the monitoring task for \dataLTLf is solvable.
\end{theorem}

\paragraph{II. Acyclic signature and monotonicity constraints.}
Next, we consider \dataLTLf with arithmetic.
To obtain decidability, we make multiple restrictions. First, we assume that the signature $\Sigma$ is \emph{tame}, 
i.e., in the sort graph of $\Sigma$, sorts $\mathit{rat}$ and $\mathit{int}$ are leaves.
In this case, the combined theory $\TT$ is known to enjoy
model completion ~\cite[Thm. 7]{CalvaneseGGMR22}.
Second, we require that $\Sigma$ is acyclic.
Third, we restrict arithmetic constraints to
\emph{monotonicity constraints} (MCs), i.e., constraints of the form $t \odot t'$ where $t,t'$ are $\Sigma$-terms of sort $rat$ over free variables $V$ 
and $\odot$ is one of $=$, $\neq$, $\leq$, or $<$.
MCs have been repeatedly considered in the verification literature~\cite{DD07,FLM19} and are important in BPM since transition guards of this shape can be learned automatically from data~\cite{LeoniA13}.
We call the class of properties that adhere to these restrictions \dataLTLf-MC.
For instance, the properties of Ex.~\ref{ex:concert} are in this class.
We obtain the following result, subsuming \cite[Thm. ~20]{FMPW23}:

\begin{restatable}{theorem}{theoremacyclicMC}
\label{thm:MC:acyclic}
The monitoring task is solvable for \dataLTLf-MC properties.
\end{restatable}
\begin{proof}[Proof (idea)]
    As shown in \cite{GianolaMW24}, under the assumptions of tame signature and MCs, the procedure $\m{TameCombCover}$ \cite[Sec. 8]{CalvaneseGGMR22} has finitely many possible results of the form $\phi_1(\vec X) \wedge \phi_2(\vec Y, \vec t(\vec X))$, up to equivalence. Hence there are finitely many regression constraints up to equivalence, so the coreachability graph construction in Def.~\ref{def:cg} terminates.
\end{proof}

\paragraph{III. Local Finiteness.}
Next, we consider decidability beyond acyclic signatures for \dataLTLf without arithmetic, using the concept of local finiteness~\cite{Lipparini82}.
A $\Sigma$-theory $\TT$ is called \emph{$k$-locally finite}, for some $k\,{\in}\,\mathbb N$, if  for every finite set of variables $X$, the number of $\Sigma$-terms with free variables $X$ is upper-bounded by $k$, up to $\TT$-equivalence.
As  local finiteness directly ensures that there are only finitely many non-equivalent quantifier-free formulas in $\TT^*$, we get:

\begin{theorem}
\label{thm:locally:finite}
If $\TT$ is a $k$-locally finite $\Sigma$-theory that admits model completion, the monitoring task is solvable for any property $\psi$ over $\Sigma$.
\end{theorem}

One way to ensure in Thm.~\ref{thm:locally:finite} that $\TT$ has a model completion 
is to restrict axioms in $\TT$ to single-variable 
sentences \cite{CalvaneseGGMR20}:

\begin{corollary}
\label{cor:locally:finite}
The monitoring task is solvable for a \dataLTLf property $\psi$ over $\Sigma$ and
if $\TT$ is a $k$-locally finite $\Sigma$-theory axiomatized by single-variable axioms.
\end{corollary}

One can also obtain decidability by replacing acyclicity in the combination result of Thm.~\ref{thm:MC:acyclic} by local finiteness, using in 
the proof local finiteness to ensure that up to equivalence there are only finitely many formulas without arithmetic.

\paragraph{IV. Bounded lookback.}
The above criteria achieve decidability mostly by syntactic restrictions. 
However, finite coreachability graphs can also be achieved by controlling interaction of constraints in properties, as it was already shown for the arithmetic case~\cite{FMPW23}:
\emph{Bounded lookback} expresses, intuitively, that the verification obligations in an \LTLfMT property can be evaluated in a sliding window and depend on only boundedly many steps from the past; it generalizes the feedback freedom property~\cite{DDV12}.
Bounded lookback is defined via \emph{computation graphs}:
Let $w= \langle \varsigma_0, \dots, \varsigma_{n-1}\rangle\in \Theta^*$ such that $q \gotos{w} q'$ in $\NFA$, for some states $q$, $q'$.
For each $0\,{\leq}\,i\,{<}\,n$, let $V_i$ be a fresh set of variables of the same size and sorts as $V$, and $\mathcal V = \bigcup_{i=0}^{n-1} V_i$.
Consider the formulas
$\varsigma_0(V/V_0)$, i.e., the formula obtained from $\varsigma_0$ by systematically replacing $V$ with $V_0$; and $\varsigma_i(\Vpre/V_{i-1}, V/V_i)$, i.e., the formula obtained from $\varsigma_i$ by replacing $\Vpre$ with $V_{i-1}$ and $V$ with $V_i$, for $0<i<n$.
The \emph{computation graph} $G_w$ of $w$ is the undirected graph with nodes $\mathcal V$ and an edge from $u\in \mathcal V$ to $v\in \mathcal V$ iff the two variables occur in a common literal of
$\varsigma_0(V/V_0) \wedge \bigwedge_{i\geq 1} \varsigma_i(\Vpre/V_{i-1}, V/V_i)$.
Finally, $[G_{w}]$ denotes the graph obtained from $G_{w}$ by collapsing
all edges corresponding to equality literals $u=v$ for $u, v \in \mc V$.

\begin{definition}
The property $\psi$ has \emph{$K$-bounded lookback} if for all states $q$, $q'$ in the NFA $\NFA$ and
all $w \in \Theta^*$ such that $q \gotos{w} q'$,
all acyclic paths in $[G_{w}]$ have length at most $K$.
\end{definition}

E.g., $(x\,{=}\,f(y)) \U P(\prev x, x)$ has 1-bounded lookback, but the property $\varphi$ in Ex.~\ref{ex:concert} does not because it requires to compare prices of unboundedly many tickets.
We show next that the monitoring task for bounded lookback properties is solvable.
The proof shows that every formula occurring in a coreachability graph comes from a finite set of formulas since it has bounded quantifier depth and a finite vocabulary.

\begin{restatable}{theorem}{thmBL}
\label{thm:bl}
For every $K > 0$, the monitoring task is solvable for $K$-bounded-lookback properties.
\end{restatable}


\section{Conclusions}
\label{sec:conclusions}
We have studied for the first time  anticipatory monitoring of data-aware temporal properties expressed in \LTLfMT, and identified expressive fragments of the logic for which monitors can be effectively constructed. Specifically, we show monitoring to be solvable for \dataLTLf, an expressive class of first-order temporal properties where the data part combines arithmetic and references to a database.   
Our approach is fully implemented in the tool \tool, which employs CVC5~\cite{cvc5} for dealing with data (SMT checks and quantifier elimination).
Source code and web interface are available via \emph{https://monitoring.adatool.dev}.
The tool takes as input a property, and a trace (given by a sequence of assignments and a set of atoms to define the model); as output,
\tool reports the monitoring state.
Preliminary experiments are reported in the appendix, leaving for future work a more extensive empirical evaluation.
Another interesting continuation is to consider  dynamic systems where the data component \emph{partially changes} over time. This is  essential to capture controlled updates over a database, like updating the price of a ticket in our running example. 

\noindent

\clearpage
\bibliographystyle{named}
\bibliography{references}

\clearpage
\longversion{\appendix
\section{Automata}

First, we define the function $\delta$ used in the definition of the NFA $\NFA$.

To that end, 
let an element of the alphabet $\varsigma\inn \Theta$ be \emph{satisfiable} if the conjunction of formulas in $\varsigma$ is $\TT$-satisfiable (this includes that $\{\lambda, \neg \lambda \}\not \subseteq \varsigma$).
For instance, if $C = \{x\,{>}\,0, x\,{=}\,3\}$ then the symbols $\{x\,{>}\,0, x\,{=}\,3\}$ and $\{x\,{\leq}\,0, \lambda\}$ are satisfiable, but $\{x\,{>}\,0, \lambda, \neg \lambda\}$ and 
$\{x\,{=}\,3, x\,{\leq}\,0\}$ are not.
Note that by Assumptions~\ref{ass:T}  and the restriction that first-order formulas in \LTLf properties are quantifier-free, satisfiability is decidable.
We use again $\LL$ for the set of such properties.

For two sets of such pairs $R_1$, $R_2$, let
$R_1 \owedge R_2 = \{ (\inquotes{\psi_1 \wedge \psi_2}, \varsigma_1 \cup \varsigma_2) \mid (\inquotes{\psi_1}, \varsigma_1) \inn R_1, (\inquotes{\psi_2}, \varsigma_2) \inn R_2\text{ and }\varsigma_1\cup \varsigma_2\text{ is satisfiable} \}$, where we simplify $\psi_1 \wedge \psi_2$ if possible; and let $R_1 \ovee R_2$ be defined in the same way, replacing con- with disjunction.

\begin{definition}
For $\psi \in \LL \cup \{\top, \bot\}$, $\delta$ is as follows:

\begin{tabular}{@{~}r@{~}l}
$\delta(\inquotes{\top})$ =& $\{(\inquotes{\top},\emptyset)\} \text{ and }\delta(\inquotes{\bot}) = \{(\inquotes{\bot},\emptyset)\}$\\
$\delta(\inquotes{\phi})$ =& $\{(\inquotes{\top},\{\phi\}),(\inquotes{\bot},\{\neg\phi\})\}$  \\
$\delta(\inquotes{\psi_1 \wedge \psi_2})$ =&  $\delta(\inquotes{\psi_1}) \owedge \delta(\inquotes{\psi_2})$\\
$\delta(\inquotes{\psi_1 \vee \psi_2})$ =& 
  $\delta(\inquotes{\psi_1}) \ovee \delta(\inquotes{\psi_2})$ \\
  $\delta(\inquotes{\X \psi})$ =& 
  $\{(\inquotes{\psi},\{\neg \lambda\}), (\inquotes{\bot}, \{\lambda\})\}$ \\
  $\delta(\inquotes{\wX \psi})$ =& 
  $\{(\inquotes{\psi},\{\neg \lambda\}), (\inquotes{\top}, \{\lambda\})\}$\\
$\delta(\inquotes{\psi_1 \U \psi_2})$ =&
  $\delta(\inquotes{\psi_2}) \ovee (\delta(\inquotes{\psi_1})
  \owedge \delta(\inquotes{\sX(\psi_1 \U \psi_2)}))$ \\
$\delta(\inquotes{\psi_1 \R \psi_2})$ =&
  $\delta(\inquotes{\psi_2}) \owedge (\delta(\inquotes{\psi_1})
  \ovee \delta(\inquotes{\wX(\psi_1 \R \psi_2)}))$.
\end{tabular}

where $\phi$ is a FO formula.
\end{definition}

To refine the notion of consistency to individual trace instants, we use:

\begin{definition}
A symbol $\varsigma \in \Theta$ is \emph{consistent with instant $i$} of trace $\tau = (M, \langle\alpha_0, \dots, \alpha_{n-1}\rangle)$ (or simply with $(\tau, i)$)
if 
(1) $0\,{\leq}\,i\,{<}\,n{-}1$ and $\lambda \not\in \varsigma$, or $i=n{-}1$ and $\neg\lambda \not\in \varsigma$, 
(2) if $i\,{=}\,0$ then $\varsigma$ does not contain $\Vpre$ and $M, \alpha_0 \models \nolambda{\varsigma}$; 
and if $i\,{>}\,0$ then $M, \combine{\alpha_{i-1}}{\alpha_{i}} \models \nolambda{\varsigma}$.
\end{definition}

We next show that $\delta$ is total in the sense that it provides a matching transition for every instant of a trace. 

\begin{lemma}
\label{lem:delta:total}
Consider $\psi \in \LL\cup \{\top,\bot\}$, a trace $\tau$ of length $n$, and $0\leq i < n$ such that if $i=0$ then $\psi$ has safe lookback.
Then there is some $(\psi', \varsigma) \inn \delta(\psi)$ 
such that 
$\varsigma$ is consistent with $(\tau, i)$.
\end{lemma}
\begin{proof}
By structural induction on $\psi$. The claim is easy to check for every base
case of the definition of $\delta$, and in all other cases it follows from the induction hypothesis and the definition of $\owedge$ and $\ovee$.
\end{proof}

The next result show that $\delta$ preserves and reflects satisfaction of properties at trace instants.

\begin{lemma}
\label{lem:delta}
Let $\psi \in \LL$,
$\tau = (M,\langle\alpha_0, \dots, \alpha_{n-1}\rangle)$ a trace, and $0\,{\leq}\,i\,{<}\,n$.
Then the following are equivalent:
\begin{compactitem}
\item
$\tau,i \models \psi$, and if $i=0$ then $\psi$ has safe lookback, and
\item
there is some $(\psi', \varsigma)\in \delta(\psi)$ such that\\
\noindent
\begin{tabular}{@{\ }r@{\ }p{8cm}}
$(a)$ & $\varsigma$ is consistent with instant $i$ of $\tau$, and\\
$(b)$ & if $i\,{=}\,n{-}1$ then $\psi'\,{=}\,\top$, and otherwise $\tau,i{+}1 \models \psi'$.
\end{tabular}
\end{compactitem}
\end{lemma}
\begin{proof}
We prove both directions simultaneously by induction on $\psi$, using the definition of $\delta$.
\begin{compactitem}
\item
If $\psi = \top$ then it must be $\psi'=\top$ and $\varsigma=\emptyset$, so
both directions hold trivially.
If $\psi = \bot$ it must be $\psi'=\bot$, so neither $\tau,i\models \psi$
nor $\tau,i+1\models \psi'$ hold.
\item 
Let $\psi$ be a first-order formula $c$. 
$(\Longrightarrow)$
Suppose $\tau,i\models c$. For $(\top, \{c\}) \in \delta(\inquotes{c})$, (b) holds because $\psi' = \top$, and (a) holds
because $\tau,i\models c$ and the assumption that $c$ has safe lookback if $i=0$ implies $\lambda$-consistency of $\{c\}$.

$(\Longleftarrow)$
Suppose there is some 
$(\psi', \varsigma)\in \delta(c)$ such that (a) and (b) hold. By definition of $\delta$, $\psi'$ is either $\top$ or $\bot$, but the latter can be excluded, so that $[\varsigma] = \{c\}$.
By consistency, we have either $i=0$, $\varsigma$ does not mention $\Vpre$, and $M, \alpha_0 \models \nolambda{\varsigma}$, or $i > 0$ and $M,\combine{\alpha_{i-1}}{\alpha_{i}} \models  \nolambda{\varsigma} = c$, so  $\tau,i\models c$, and $c$ is well-defined at instant $i$ of $\tau$, so safe lookback holds if $i=0$.
\item
We have $\tau, i \models \psi_1 \wedge \psi_2$ iff $\tau, i \models \psi_1$ and
$\tau, i \models \psi_2$.
By the induction hypothesis, this is the case iff there are some $(\inquotes{\psi_1'}, \varsigma_1) \in \delta(\inquotes{\psi_1})$ and $(\inquotes{\psi_2'}, \varsigma_2) \in \delta(\inquotes{\psi_2})$
such that $\varsigma_1$, $\varsigma_2$ are consistent with $\tau$ at $i$;
and either $i<n$ and $\tau,i+1 \models \psi_1'$ and $\tau,i+1 \models \psi_2'$, or
$\psi_1' = \psi_2' = \top$.
We have $(\inquotes{\psi_1' \wedge \psi_2'}, \varsigma_1 \cup \varsigma_2) \in \delta(\inquotes{\psi_1 \wedge \psi_2})$,
where
$\varsigma_1 \cup \varsigma_2$ is consistent with $\tau$ at $i$ iff both $\varsigma_1$ and $\varsigma_2$ are.
Moreover, if $i<n$ then $\tau,i+1 \models \psi_1'\wedge \psi_2'$ iff $\tau,i+1 \models \psi_1'$ and $\tau,i+1 \models \psi_2'$ hold; and if $i=n$ then $\psi_1' \wedge \psi_2' = \top$ iff $\psi_1' = \psi_2' = \top$.
In addition, $\psi_1 \wedge \psi_2$ has safe lookback iff this holds for both $\psi_1$ and $\psi_2$.
\item
($\Longrightarrow$)
Let $\tau, i \models \psi_1 \vee \psi_2$, so $\tau, i \models \psi_1$ or
$\tau, i \models \psi_2$. 
Assume the former.
Note that $\psi_1 \vee \psi_2$ has safe lookback iff this holds for both $\psi_1$ and $\psi_2$.
By the induction hypothesis, this is the case if there are some $(\inquotes{\psi_1'}, \varsigma_1) \in \delta(\inquotes{\psi_1})$
such that $\varsigma_1$ is consistent with $\tau$ at $i$;
and either $i<n$ and $\tau,i+1 \models \psi_1'$, or
$\psi_1' = \top$.
By Lem.~\ref{lem:delta:total},
there is some $(\inquotes{\psi_2'}, \varsigma_2) \in \delta(\inquotes{\psi_2})$ such that $\varsigma_1$ is consistent with $\tau$ at $i$;
We have $(\inquotes{\psi_1' \vee \psi_2'}, \varsigma_1 \cup \varsigma_2) \in \delta(\inquotes{\psi_1 \vee \psi_2})$,
where
$\varsigma_1 \cup \varsigma_2$ is consistent with $\tau$ at $i$ iff both $\varsigma_1$ and $\varsigma_2$ are.
Moreover, if $i<n$ then $\tau,i+1 \models \psi_1'\vee \psi_2'$ iff $\tau,i+1 \models \psi_1'$ or $\tau,i+1 \models \psi_2'$ hold; and if $i=n$ then $\psi_1' \vee \psi_2' = \top$ iff $\psi_1' = \top $ or $\psi_2' = \top$.

($\Longleftarrow$)
There must be some $(\inquotes{\psi_1' \vee \psi_2'}, \varsigma_1 \cup \varsigma_2) \in \delta(\inquotes{\psi_1 \vee \psi_2})$,
such that
$\varsigma_1 \cup \varsigma_2$ is consistent with $\tau$ at $i$, so both $\varsigma_1$ and $\varsigma_2$ are.
Moreover, if $i<n$ and $\tau,i+1 \models \psi_1'\vee \psi_2'$ then $\tau,i+1 \models \psi_1'$ or $\tau,i+1 \models \psi_2'$ hold; and if $i=n$ and $\psi_1' \vee \psi_2' = \top$ then $\psi_1' = \top $ or $\psi_2' = \top$.
Suppose $\tau,i+1 \models \psi_1'$ resp. $\psi_1' = \top $.
By the induction hypothesis, $\tau, i \models \psi_1$. Since $\varsigma_1 \cup \varsigma_2$ is consistent with $\tau$ at $i$, $\psi_1$ and $\psi_2$ have safe lookback, so $\tau, i \models \psi_1 \vee \psi_2$.
\item
($\Longrightarrow$)
If $\tau, i \models \sX \psi$ then $i<n{-}1$ and $\tau, i+1 \models \psi$.
As $(\inquotes{\psi}, \{\neg \lambda\}) \in \delta(\inquotes{\sX \psi})$ and $\{\neg \lambda\}$ is
consistent with $\tau$ at $i$ because $i<n-1$, the claim holds.
($\Longleftarrow$)
If $(\inquotes{\psi'}, \varsigma) \in \delta(\inquotes{\sX \psi})$ such that
$\varsigma$ is consistent and $\tau, i+1 \models \psi'$ or $\psi' = \top$, by the 
definition of $\delta$ it must be $\psi' = \psi$ and $\varsigma = \{\neg\lambda\}$.
By consistency, $i<n-1$, so $\tau, i+1 \models \psi$, and hence $\tau, i \models \sX \psi$.
\item
($\Longrightarrow$)
If $\tau, i \models \wX \psi$ then $i=n-1$, or $i<n-1$ and $\tau, i+1 \models \psi$.
In the latter case, we reason as for $\sX$.
If $i=n$, we take $(\inquotes{\top}, \{\lambda\}) \in \delta(\inquotes{\wX \psi})$ and $\{\lambda\}$ is
consistent with $\tau$ at $i$, so the claim holds.
($\Longleftarrow$)
If $(\inquotes{\psi'}, \varsigma) \in \delta(\inquotes{\wX \psi})$ such that
$\varsigma$ is consistent and $\tau, i+1 \models \psi'$ or $\psi' = \top$, by the 
definition of $\delta$ it must be either $\psi' = \psi$ and $\varsigma = \{\neg\lambda\}$, or $\psi' = \top$ and $\varsigma = \{\lambda\}$.
In the former case, we reason as for $\sX$, otherwise, $\tau, i \models \wX \psi$ holds anyway.
\item
($\Longrightarrow$)
If $\tau, i \models \psi_1 \U \psi_2$, $\psi_1 \U \psi_2$ has safe lookback; and either (i) $\tau, i \models \psi_2$, or (ii) $i<n{-}1$,
$\tau, i \models \psi_1$, and $\tau, i+1 \models \psi_1 \U \psi_2$.
In case (i), by the induction hypothesis there is some  $(\psi', \varsigma') \in \delta(\inquotes{\psi_2})$
such that (a) and (b) hold.
By (a), $\varsigma'$ is consistent with $\tau$ at $i$. By Lem.~\ref{lem:delta:total},
there are some  $(\psi_1', \varsigma_1) \in \delta(\inquotes{\psi_1})$ and $(\psi_2', \varsigma_2) \in \delta(\inquotes{\sX (\psi_1 \U \psi_2)})$ such that $\varsigma_1$ and $\varsigma_2$ are consistent with $\tau$ at $i$. We have that $(\inquotes{\psi}, \varsigma) \in \delta(\inquotes{\psi_1 \U \psi_2})$ for $\psi = \psi' \vee (\psi_1' \wedge \psi_2')$ and $\varsigma = \varsigma' \cup \varsigma_1 \cup \varsigma_2$, and $\varsigma$ is consistent with $\tau$ at $i$. Moreover, by (b) either $i<n{-}1$ and $\tau,i+1 \models \psi'$, or $i=n{-}1$ and 
$\psi' = \top$. Thus, also $i<n{-}1$ and $\tau,i+1 \models \psi$, or $i=n{-}1$ and 
$\psi = \top$ hold.
In case (ii) we have $i<n{-}1$. As safe lookback holds for $\psi_1$ if $i=0$, by the induction hypothesis there is some  $(\psi_1', \varsigma_1') \in \delta(\psi_1)$
such that (a) and (b) hold, so as $i<n{-}1$ it is $\tau, i+1 \models \psi_1'$ ($\star$). Moreover,  $(\inquotes{\psi_1 \U \psi_2}, \{\neg \lambda\}) \in \delta(\inquotes{\sX (\psi_1 \U \psi_2)})$ and $\{\neg \lambda\}$ is
consistent with $\tau$ at $i$ because $i<n{-}1$.
By Lem.~\ref{lem:delta:total}, there is also some  $(\psi_2', \varsigma_2') \in \delta(\psi_2)$ such that $\varsigma_2'$ is consistent with $\tau$ at $i$.
Now, we have $(\inquotes{\psi}, \varsigma) \in \delta(\inquotes{\psi_1 \U \psi_2})$ for $\psi = \psi_2' \vee (\psi_1' \wedge (\psi_1 \U \psi_2))$ and $\varsigma = \varsigma' \cup \varsigma_1 \cup \varsigma_2$.
Then $\varsigma$ as a union is also consistent with $\tau$ at $i$, and $\tau, i+1 \models \psi$ because of ($\star$) and the assumption of case (ii).

($\Longleftarrow$) Suppose there is some $(\inquotes{\psi}, \varsigma) \in \delta(\inquotes{\psi_1 \U \psi_2})$ such that (a) and (b) holds.
By the definition of $\delta$, we can write $\psi$ and $\varsigma$ as
$\psi = \psi_2' \vee (\psi_1' \wedge \psi_3')$ and $\varsigma = \varsigma_1 \cup \varsigma_2 \cup \varsigma_3$, where $(\inquotes{\psi_2'}, \varsigma_2) \in \delta(\inquotes{\psi_2})$,
$(\inquotes{\psi_1'}, \varsigma_1) \in \delta(\inquotes{\psi_1})$, and
$(\inquotes{\psi_3'}, \varsigma_3) \in \delta(\inquotes{\sX (\psi_1 \U \psi_2)})$.
By (a), all of $\varsigma_1$, $\varsigma_2$, $\varsigma_3$ must be consistent with $(\tau,i)$.
By (b), either $i<n{-}1$ and $\tau,i+1 \models \psi_2' \vee (\psi_1' \wedge \psi_3')$, or
$i=n{-}1$ and $\psi_2' \vee (\psi_1' \wedge \psi_3') = \top$.
We distinguish these two cases: if $i=n{-}1$ then either $\psi_2' = \top$ or $\psi_1' \wedge \psi_3' = \top$, but the latter is impossible
because it would imply $\psi_3' = \top$, but by consistency this is not possible if $i=n$.
So $\psi_2' = \top$. By the induction hypothesis, $\tau, i \models \psi_2$, and hence $\tau, i \models \psi_1 \U \psi_2$.
If $i<n$ then $\tau,i+1 \models \psi_2' \vee (\psi_1' \wedge \psi_3')$, so either
$\tau,i+1 \models \psi_2'$ or $\tau,i+1 \models \psi_1' \wedge \psi_3'$.
In the former case, by the induction hypothesis $\tau, i \models \psi_2$, and hence $\tau, i \models \psi_1 \U \psi_2$. In the latter case, $\tau,i+1 \models \psi_1'$ and $\tau,i+1 \models \psi_3'$.
By the induction hypothesis, $\tau,i \models \psi_1$. By the definition of $\delta$, we must have
$\psi_3' = \psi_1 \U \psi_2$. As $\tau,i \models \psi_1$ and $\tau,i+1 \models \psi_1 \U \psi_2$, we obtain again $\tau,i \models \psi_1 \U \psi_2$.
\item The case for $\R$ is similar.
\qedhere
\end{compactitem}
\end{proof}

\begin{lemma}
\label{lem:deltastar}
Let $\psi$ have safe lookback.
Then $\tau \models \psi$
iff there is a  word $w\in \Theta^*$ that is consistent with a trace $\tau$ such that $\inquotes{\top} \in \delta^*(\inquotes{\psi},w)$.
\end{lemma}
\begin{proof}
($\Longrightarrow$)
Suppose that $\tau \models \psi$, i.e., $\tau,0 \models \psi$.
We show that, more generally, for all $i$, $0\,{\leq}\,i\,{<}\,n$,
and every property $\chi \in \LL$,
if $\tau, i \models \chi$ and $\chi$ has safe lookback in the case where $i=0$, then 
there is a word $w_i = \langle\varsigma_i, \varsigma_{i+1}, \dots, \varsigma_{n-1}\rangle$
such that
$\inquotes{\top} \in \delta^*(\inquotes{\chi},w_i)$,
and 
$\varsigma_j$ is consistent with $(\tau,j)$ for all $j$, $i\,{\leq}\,j\,{<}\,n$.
The proof is by induction on $n-i$.
In the base case where $i=n{-}1$, we assume that $\tau,n \models \chi$. 
By Lem. \ref{lem:delta} there is some $\varsigma_n$ such that
$(\inquotes{\top}, \varsigma_n)\in \delta(\inquotes{\chi})$,
and $\varsigma_n$ is consistent with $(\tau,n)$.
For the induction step, assume $i\,{<}\,n-1$ and $\tau, i \models \chi$.
By Lem. \ref{lem:delta} there is some
$(\inquotes{\chi'}, \varsigma_i) \in \delta(\inquotes{\chi})$ such that
$\tau, i{+}1 \models \chi'$, and
$\varsigma_i$ is consistent with $(\tau, i)$.
By the induction hypothesis, 
there is a word $w_{i+1} = \langle\varsigma_{i+1}, \dots, \varsigma_{n-1}\rangle$
such that
$\inquotes{\top} \in \delta^*(\inquotes{\chi'},w_{i+1})$,
and 
$\varsigma_j$ is consistent with $(\tau,j)$ for all $j$, $i\,{<}\,j\,{<}\, n$.
Thus, we can define $w_{i} = \langle\varsigma_i,\varsigma_{i+1}, \dots, \varsigma_{n}\rangle$,
which satisfies
$\inquotes{\top} \in \delta^*(\inquotes{\chi},w_{i})$ and 
$\varsigma_j$ is consistent with $\tau$ at instant $j$ for all $j$, $i\,{\leq}\,j\,{<}\, n$.
This concludes the induction step.
By assumption, $\tau,0 \models \psi$  holds, and
$\psi$ has safe lookback.
From the case $i = 0$ of the above statement, we obtain a word $w$
such that
$\inquotes{\top} \in \delta^*(\inquotes{\psi},w)$ and $w$ is
consistent with $\tau$.

($\Longleftarrow$)
Let $w =\langle \varsigma_0, \dots, \varsigma_{n-1}\rangle$ be consistent with $\tau = \langle\alpha_0,\dots, \alpha_{n-1}\rangle$.
If $\inquotes{\top} \in \delta^*(\inquotes{\psi},w)$, there must be properties
$\chi_0,\chi_1,\dots, \chi_{n}$ such that $\chi_0 = \psi$, $\chi_{n} = \top$,
and $(\inquotes{\chi_{i{+}1}}, \varsigma_{i}) \in \delta(\inquotes{\chi_i})$
for all $i$, $0\leq i < n$.
As $w$ is consistent with $\tau$, 
$\varsigma_i$ is consistent with 
$(\tau, i)$ for all $i$, $0\,{\leq}\,i\,{<}\,n$.
In order to show that $\tau \models \psi$, we verify that
$\tau, i \models \chi_i$
for all $i$, $0\,{\leq}\,i\,{<}\,n$.
The reasoning is by induction on $n-i$.
In the base case $i\,{=}\,n{-}1$.
Thus $\chi_{n} = \top$ and 
$(\inquotes{\chi_{n}}, \varsigma_{n}) \in \delta(\inquotes{\chi_{n-1}})$, so that
with consistency it follows from
Lem. \ref{lem:delta} that $\tau, n-1 \models \chi_{n-1}$.
In the induction step $i\,{<}\,n{-}1$, and we assume by the induction hypothesis that $\tau, i{+}1 \models \chi_{i{+}1}$.
We have
$(\inquotes{\chi_{i{+}1}}, \varsigma_{i+1}) \in \delta(\inquotes{\chi_i})$, so
$\tau, i \models \chi_i$ follows again from Lem. \ref{lem:delta} using consistency. This concludes the induction step, and the claim follows for the case $i\,{=}\,0$ because 
$\chi_0 = \psi$.
\end{proof}

At this point, Thm.~\ref{thm:automaton:acceptance} follows 
from the definition of $\NFA$ and Lem.~\ref{lem:deltastar}.

\section{Monitoring Approach}

\lemmaregression*
\begin{proof}
\begin{compactenum}[(1)]
\item 
By induction of the length of $w$. 
If $w$ is empty, $\precond(w) = \top$, and $\langle\emptyset\rangle$ is consistent with $(M,\langle\alpha\rangle)$.

Otherwise, let $w = \langle \varsigma\rangle \cdot w'$ and 
$M,\alpha \models \precond(w)$.
Since 
\begin{align*}
\precond(w)
&= \regress(\precond(w'), \varsigma) \\
&=  \exists \vec Y. \nolambda{\varsigma}(\vec V/\vec \Vpre, \vec Y/\vec V)\wedge \precond(w')(\vec Y/\vec V)
\end{align*}
there must be an assignment $\nu$ with domain $V \cup Y$ such that $\nu(v) = \alpha(v)$ for all $v\in V$ and such that $M, \nu \models \bigwedge  \varsigma(\vec V/\vec \Vpre, \vec Y/\vec V)\wedge \precond(w')(\vec Y/\vec V)$. Let $\alpha'$ be the assignment such that $\nu(\vec Y) = \alpha'(\vec V)$, so that we have $M,\combine{\alpha'}\alpha \models \varsigma$ and  $M, \alpha' \models \precond(w')$.
By the induction hypothesis, there is a trace $\tau'$ with model $M$ such that 
$\langle\emptyset\rangle\,{\cdot}\,w'$ is consistent with $(M,\langle\alpha'\rangle)\,{\cdot}\,\tau'$ ($\star$). For $\tau = (M,\langle \alpha' \rangle) \cdot \tau'$, the word 
$\langle\emptyset\rangle\,{\cdot}\,w = \langle\emptyset,\varsigma\rangle\,{\cdot}\,w'$ is consistent with $\tau$ because $M,\alpha' \models \bigwedge \emptyset = \top$, $M, \combine{\alpha'}\alpha \models \varsigma$, and ($\star$).
\item
By induction on the length of $w$.
If $w$ is empty, then $\precond(w) = \top$ and $M,\alpha \models \top$ is trivially true.

Otherwise, let $w = \langle \varsigma\rangle \cdot w'$ and $\langle \emptyset,\varsigma\rangle \cdot w'$ be consistent with a trace of the form 
$(M,\langle \alpha\rangle) \cdot \tau$. By the definition of consistency, $\tau$ is of the same length as $w$, so let $\tau = (M,\langle \alpha'\rangle)\cdot \tau'$. By consistency, 
$M, \combine{\alpha}{\alpha'} \models \varsigma$ must hold, and moreover 
$\langle\emptyset\rangle\,{\cdot}\,w'$ must be consistent with $(M,\langle\alpha'\rangle)\,{\cdot}\,\tau'$. By the induction hypothesis, $M, \alpha' \models \precond(w')$.
It follows with $M, \combine{\alpha}{\alpha'} \models \varsigma$ that 
$M, \alpha \models \varsigma(\vec V/\Vpre, \alpha'(\vec V)/\vec V) \wedge \precond(w')(\alpha'(\vec V)/\vec V)$, so 
\begin{align*}
M, \alpha \models &\exists \vec Y. \nolambda{\varsigma}(\vec V/\Vpre, \vec Y) \wedge \precond(w')(\vec Y/\vec V)\\
&= \regress(\precond(w'), \varsigma) = \precond(w).
\end{align*}
\qedhere
\end{compactenum}
\end{proof}

\lemmacg*
\begin{proof}
\begin{compactenum}[(1)]
\item 
By induction on $w$. 
If $w$ is empty then $(q, \varphi) = (q_f, \top)$, and since $\precond(w) = \top$,
the claim holds.
Otherwise, let $w = \langle\varsigma\rangle\cdot w'$ and consider a path
$(q, \varphi) \goto{\varsigma}  (q',\varphi') \gotos{w'} (q_f, \varphi)$.
By the induction hypothesis, 
$q' \gotos{w'} q_f$ and $\varphi' \sim \precond(w')$.
By Def.~\ref{def:cg}, the edge $(q,\varphi) \goto{\varsigma} (q', \varphi')$ exists 
because $q \goto{\varsigma}q'$ in $\NFA$, and we have $\regress(\varphi', \varsigma) \sim \varphi$ and $\regress(\varphi', \varsigma)$ is $\TT$-satisfiable. 
Hence there is a transition sequence $q \gotos{w} q_f$ in $\NFA$.
Moreover, we have $\varphi \sim \regress(\varphi', \varsigma) \sim \regress(\precond(w'), \varsigma) = \precond(w)$.
By Assumptions~\ref{ass:T},  $\regress(\varphi', \varsigma)$ is also $\TT^*$-satisfiable, hence so is $\varphi$, so that again using Assumptions~\ref{ass:T} $\varphi$ is also $\TT$-satisfiable.
\item
By induction on $w$.
If $w$ is empty, $q = q_f$ and the empty path satisfies the claim.
Otherwise, let $w = \langle\varsigma\rangle\cdot w'$  and suppose 
$q  \goto{\varsigma}  q' \gotos{w'} q_f$ and $\precond(w)$ is $\TT$-satisfiable,
hence also $\TT^*$-satisfiable. Moreover, by definition of $\precond$, also $\precond(w')$ must be $\TT$-satisfiable. 
By the induction hypothesis
there is a path $(q', \varphi') \gotos{w'} (q_f,\top)$ such that $\varphi' \sim \precond(w')$.
Since $\precond(w)=\regress(\precond(w'), \varsigma) \sim \regress(\varphi', \varsigma)$, is $\TT^*$-satisfiable, so is $\varphi$ because of
$\varphi \sim \regress(\varphi', \varsigma)$.
By Assumptions~\ref{ass:T}, $\regress(\varphi', \varsigma)$ is also $\TT$-satisfiable.
Therefore, $G$ must contain a transition $(q, \varphi) \goto{\varsigma}  (q',\varphi')$, and hence a path $(q, \varphi) \gotos{w}  (q_f,\top)$.
Again, we have $\varphi \sim \regress(\varphi', \varsigma) \sim \regress(\precond(w'), \varsigma) = \precond(w)$.
\end{compactenum}
\end{proof}

\section{Decidability}

\thmBL*
\begin{proof}
Let $\psi$ have $K$-bounded lookback, for some $K>0$, and $\Phi$ be the set of formulas with free variables $V\cup V_0$, 
quantifier depth at most $K$, and using as atoms all constraints in 
$\NFA$, but where variables in $V \cup \Vpre$ may be replaced by $V$ or any of the quantified variables.
Being a set of formulae with bounded quantifier depth over a finite vocabulary, $\Psi$ is finite up to equivalence.
To prove that $\mathcal Q \times\Phi$ is a finite set that contains all possible regression constraints, 
we show that for every well-formed $w\in \Theta^*$, $\Phi$ contains a formula $\varphi$ such that $\varphi \equiv \precond(w)$.
The proof is by induction of $w$.
If $w$ is empty,  $\precond(w) = \top$ has quantifier depth 0, so there is 
a formula in $\Phi$ equivalent to $\precond(w)$ by definition of $\Phi$.
Otherwise, $w = w'\cdot\langle\varsigma\rangle$, and
$\precond(w) = \regress(\precond(w'), \varsigma)$.
By induction hypothesis there is some $\varphi' \in \Phi$ with $\varphi' \equiv \precond(w')$.
By definition of $\regress$, we have $\precond(w) \equiv \exists \vec U. \varphi'(\vec U) \wedge \chi$ for some quantifier free formula  $\chi\in \Phi$.
Let $[\varphi]$ be obtained from $\varphi$ by eliminating all equality literals $x=y$ in $\varphi$ and uniformly substituting all variables in an equivalence class by some arbitrary representative.
Since $\NFA$ has $K$-bounded lookback, and $[\varphi]$ encodes a part of
$[G_{w}]$, 
$[\varphi]$ is equivalent to a formula $\phi$ of quantifier depth at most $K$ over the same vocabulary, 
obtained by dropping irrelevant literals and existential quantifiers, so that $\phi \in \Phi$.
\end{proof}

\section{Experiments}

For the experiments, we collected example formulas from different sources:
the synthetic patterns from~\cite{SchneiderBBKT21}; suitable properties used in the model checking experiments in~\cite{GianolaMW24}, i.e., where no existential quantifiers come into play; as well as properties describing the dynamics of these systems, which stem from the VERIFAS experiments~\cite{DLV19}; and the running example from this paper.
Every property was monitored on five traces of varying length.
In Tab.~\ref{tab:exp}, we report for each domain the number of relation and function symbols, and for each property its size $|\phi|$, as well as the cumulative time spent on monitoring the five traces ($t$). We also report the cumulative time $t_{pre}$ spent on preprocessing (i.e., NFA construction) and on monitoring, respectively.
\begin{table}
\begin{tabular}{@{}l@{}lrrrr@{}}
\toprule
\quad & property & $|\phi|$ & $t$ & $t_{pre}$ & $t_{mon}$ \\
\midrule
\multicolumn{4}{@{}@{}@{}l}{concert Ex.~\ref{ex:concert} (R=0, F=10)} \\
  & $\varphi$        & 48  & 1.32  & 0.42  & 0.20 \\
  & $\varphi_{\mathtt{t123}}$ & 53  & 2.88  & 1.46  & 0.52 \\
\midrule
\multicolumn{4}{@{}@{}l}{synthpatterns (R=3 F=0)} \\
    &star pattern  &  13 & 0.83 & 0.05 & 0.25 \\
    &star pattern w/ frame &   34  &1.68 & 0.61 & 0.21 \\
    &linear pattern  &  14 & 1.09 & 0.10 & 0.41 \\
    &linear pattern w/ frame & 35 & 3.40 & 2.04 & 0.38 \\
    &triangle pattern  &  14 & 0.94 & 0.06 & 0.31 \\
    &triangle pattern w/ frame &   35 & 3.37 & 2.02 & 0.38 \\
\midrule 
\multicolumn{4}{@{}l}{VERIFAS lasertec (R=1, F=6)} \\
  & state range & 12  & 0.16   & 0.02   & 0.06 \\
  & DDS structure                            & 65  & 4.15   & 0.32   & 0.65 \\
  & DDS dynamics                             & 101 & 158.08 & 144.69 & 1.26 \\
\midrule
  \multicolumn{4}{@{}l}{VERIFAS book writing} (R=2 F=4) \\
 &   partial system dynamics &   75 & 2.04 & 1.81 & 0.05 \\
 &   model checking 1 &   9 & 0.02 & 0.00 & 0.01 \\
 &   model checking 2 &   9 & 0.02 & 0.00 & 0.01 \\
 &   model checking 3 &   12 & 0.02 & 0.00 & 0.01 \\
 &   model checking 4 &   9 & 0.02 & 0.00 & 0.00 \\
\midrule   
  \multicolumn{4}{@{}l}{webshop (R=2 F=7)} \\
 &   DDS structure &   85 & 6.82 & 0.30 & 0.04 \\
 &   DDS dynamics &   143 & 121.04 & 104.06 & 1.83 \\
 &   model checking 1 &   8 & 0.17 & 0.02 & 0.04 \\
 &   model checking 2 &   10 & 0.20 & 0.02 & 0.05 \\
 &   model checking 3 &   4 & 0.15 & 0.01 & 0.04 \\
\bottomrule
\end{tabular}
\caption{Monitoring experiments\label{tab:exp}}
\end{table}}{}

\end{document}